\documentclass{new_tlp}
\usepackage{mathptmx}

  \title
        {Anytime Computation of Cautious Consequences in Answer Set Programming }

  \author[M. Alviano, C. Dodaro and F. Ricca]
         {MARIO ALVIANO, CARMINE DODARO and FRANCESCO RICCA\\
         Department of Mathematics and Computer Science, 
         University of Calabria, 87036 Rende (CS), Italy\\
         \email{\{alviano,dodaro,ricca\}@mat.unical.it}}

\usepackage[ruled,linesnumbered,vlined]{algorithm2e}
\usepackage{color}
\usepackage{comment}
\usepackage{tikz}
\usepackage{xspace}
\usepackage{subfigure}

\newtheorem{theorem}{Theorem}
\newtheorem{lemma}{Lemma}
\newtheorem{example}{Example}

\newcommand{\tuple}[1]{\ensuremath{\left \langle #1 \right \rangle }}

\newcommand{\Alg}[1]{$A#1$\xspace}
\newcommand{\Any}[1]{$A#1^{*}$\xspace}

\def\naf{\ensuremath{\raise.17ex\hbox{\ensuremath{\scriptstyle\mathtt{\sim}}}}\xspace}
\def\A{\ensuremath{\mathcal{A}}\xspace}
\def\t{\ensuremath{\mathbf{t}}\xspace}
\def\u{\ensuremath{\mathbf{u}}\xspace}
\def\f{\ensuremath{\mathbf{f}}\xspace}

\begin{document}

\label{firstpage}

\maketitle

  \begin{abstract}
%


Query answering in Answer Set Programming (ASP) is usually solved by computing (a subset of) the cautious consequences of a logic program. 
This task is computationally very hard, and there are programs for which computing cautious consequences is not viable in reasonable time.
However, current ASP solvers produce the (whole) set of cautious consequences only at the end of their computation.
This paper reports on strategies for computing cautious consequences, also introducing anytime algorithms able to produce sound answers during the computation.

\textbf{To appear in Theory and Practice of Logic Programming (TPLP).}
  \end{abstract}


\section{Introduction}\label{sec:introduction}

Answer Set Programming (ASP) is a declarative language for knowledge representation and reasoning \cite{niem-99,mare-trus-99,lifs-2002,bara-2002,gelf-kakl-2014}.
In ASP, knowledge concerning an application domain is encoded by a logic program whose semantics is given by a set of stable models \cite{gelf-lifs-91}, also referred to as answer sets.
As for other nonmonotonic formalisms, the resulting knowledge base can be queried according to two possible modes of reasoning, usually referred to as brave (or credulous) and cautious (or skeptical).
Brave reasoning provides answers to the input query that are witnessed by some stable model of the knowledge base.
For cautious reasoning, instead, answers have to be witnessed by all stable models.
Cautious reasoning over ASP knowledge bases has relevant applications in various fields ranging from  Databases to Artificial Intelligence.
Among them are consistent query answering \cite{aren-etal-2003-tplp}, data integration \cite{eite-2005-lpnmr}, and ontology-based reasoning \cite{eite-etal-2008-aij}.

A common practice in ASP is to reduce query answering to the computation of a subset of the cautious consequences of a logic program \cite{leon-etal-2002-dlv}, where cautious consequences are atoms belonging to all stable models.
As an example, in the context of Consistent Query Answering (CQA), consider an inconsistent database $D$
where in relation $R = \{\tuple{1,1,1}, \tuple{1,2,1},$ $\tuple{2,2,2}, \tuple{2,2,3}, \tuple{3,2,2}, \tuple{3,3,3} \}$ 
the second argument is required to functionally depend on the first. 
Given a query $q$ over $D$, CQA amounts to computing answers of $q$ that are true in all repairs of the original database.
Roughly, a repair is a revision of the original database that is maximal and satisfies its integrity constraints.
In the example, repairs can be modeled by the following ASP rules:
\begin{eqnarray}
 R_{out}(X,Y_1,Z_1) & \leftarrow & R(X,Y_1,Z_1),\ R(X,Y_2,Z_2),\ Y_1 \neq Y_2,\ \naf R_{out}(X,Y_2,Z_2)\label{eq:rout}\\
 R_{in}(X,Y,Z) & \leftarrow & R(X,Y,Z),\ \naf R_{out}(X,Y,Z)\label{eq:rin}
\end{eqnarray}
where (\ref{eq:rout}) detects inconsistent pairs of tuples and guesses tuples to remove in order to restore consistency, while (\ref{eq:rin}) defines the repaired relation as the set of tuples that have not been removed.
The first and third arguments of $R$ can thus be retrieved by means of the following query rule:
\begin{equation}
 Q(X,Z) \leftarrow R_{in}(X,Y,Z)
\end{equation}
whose consistent answers are tuples of the form $\tuple{x,z}$ such that $Q(x,z)$ belongs to all stable models.
In this case the answer is $\{\tuple{1,1}, \tuple{2,2}, \tuple{2,3}\}$.


Cautious reasoning has been implemented by two ASP solvers, namely DLV \cite{leon-etal-2002-dlv} and clasp \cite{gebs-etal-2012-aij}, as a variant of their stable model search algorithms. 
In a nutshell, DLV and clasp compute stable models of a given program by means of a two-phase process.
The first phase is actually implemented by a possibly external instantiator producing a ground version of the input program.
The ground program is then processed by the second phase, which actually searches for stable models. 
Cautious reasoning can be obtained by reiterating the stable model search step according to a specific solving strategy. 
The procedure implemented by DLV searches for stable models and computes their intersection, which eventually results in the set of cautious consequences of the input program.
At each step of the computation, the intersection of the identified stable models represents an overestimate of the solution, which however is not provided as output by DLV.
The procedure implemented by clasp is similar, but the overestimate is outputted and used to further constrain the computation.
In fact, the overestimate is considered a constraint of the logic program, so that the next computed stable model is guaranteed to improve the current overestimate.

It is important to note that cautious reasoning is a resource demanding task, which is often not affordable to complete in reasonable time.
As a matter of fact, in these cases current ASP solvers do not produce any sound cautious consequence, as also the overestimate produced by clasp can only guarantee that some atoms do not belong to the solution.
It is interesting to observe that query answering is addressed differently in other logic programming languages.
For example, Prolog queries having infinitely many answers are common due to the presence of uninterpreted function symbols.
Prolog systems are thus designed to produce underestimates of the complete, possibly infinite solution, which actually represent sound answers to the input query.
In fact, underestimates are useful in practice, especially 
in the cases in which waiting for termination is not affordable, and this may be the case even if termination is guaranteed.
It is thus natural to ask whether underestimates can be computed also in the context of ASP.

The paper provides insights in this respect, showing that 
underestimates can actually be obtained by improving algorithms employed by ASP systems, or adapting to ASP the \emph{iterative consistency testing}
algorithm for computing backbones of propositional theories \cite{DBLP:conf/ecai/Marques-SilvaJL10}.
The paper also introduces a modified version of this last algorithm that takes advantage of restarts and heuristic values 
for faster improvement of underestimates.
An interesting aspect of the algorithms analyzed in this paper is that underestimates are produced during the computation of the complete solution.
The computation can thus be stopped either when a sufficient number of cautious consequences have been produced, or when no new answer is produced after a specified amount of time.
Such algorithms are referred to as anytime in the literature.
The empirical comparison of these algorithms highlights that they could be combined in a parallel implementation to converge faster to the complete solution.
Actually, a proof-of-concept implementation of the parallel approach is also presented in the paper to confirm this conjecture.

\section{Preliminaries}\label{sec:pre}

Syntax and semantics of propositional ASP programs are briefly introduced in this section.
A quick overview of the main steps of stable model search is also reported in order to provide the reader with essential knowledge on a process that is used but not substantially modified by the algorithms analyzed in this paper.
(For complementary introductory material on ASP see \citeNP{gelf-lifs-91,bara-2002,gelf-kakl-2014}.)

\smallskip
\noindent {\em Syntax.}
A normal logic program consists of a set of rules of the following form:
\begin{equation}\label{eq:rule}
a_0 \leftarrow a_1, \ldots, a_m, \naf a_{m+1}, \ldots, \naf a_n
\end{equation}
where each $a_i$ ($i = 0,\ldots,n$) is a propositional atom in a fixed, countable set \A, \naf denotes \emph{negation as failure}, and $n \geq m \geq 0$.
For a rule $r$ of the form (\ref{eq:rule}), atom $a_0$ is called \emph{head} of $r$, denoted $H(r)$;
conjunction $a_1, \ldots, a_m, \naf a_{m+1}, \ldots, \naf a_n$ is named \emph{body} of $r$; 
sets $\{a_1, \ldots, a_m\}$ and $\{a_{m+1}, \ldots, a_n\}$ are denoted $B^+(r)$ and $B^-(r)$, respectively.
A \emph{constraint} is a rule of the form (\ref{eq:rule}) such that $a_0 = \bot$, where $\bot$ is a fixed atom in \A. 

\smallskip
\noindent {\em Semantics.}
An interpretation $I$ is a subset of $\A \setminus \{\bot\}$.
$I$ is a model of a rule $r$, denoted $I \models r$, if $H(r) \in I$ whenever $B^+(r) \subseteq I$ and $B^-(r) \cap I = \emptyset$.
It is a model of a program $P$, denoted $I \models P$, if it is a model of all rules in $P$.
The definition of stable model is based on a notion of program reduct \cite{gelf-lifs-91}:
Let $P$ be a normal logic program, and $I$ an interpretation.
The reduct of $P$ w.r.t.\ $I$, denoted $P^I$, is obtained from $P$ by deleting each rule $r$ such that $B^-(r) \cap I \neq \emptyset$, and removing negated atoms in the remaining rules.
An interpretation $I$ is a stable model of $P$ if $I \models P^I$ and there is no $J \subset I$ such that $J \models P^I$.
Let $SM(P)$ denote the set of stable models of $P$.
If $SM(P) \neq \emptyset$ then $P$ is coherent. 
An atom $a \in \A$ is a cautious consequence of a program $P$ if $a$ belongs to all stable models of $P$.
The set of cautious consequences of $P$ is denoted $CC(P)$.

\begin{example}
The following is a (ground) program equivalent to the one reported in Section~\ref{sec:introduction}:

\smallskip
\noindent
 \begin{minipage}{.31\textwidth}
    $R_{out}\!(1,1,1) \leftarrow \naf R_{out}\!(1,2,1)$;\\
    $R_{out}\!(1,2,1) \leftarrow \naf R_{out}\!(1,1,1)$;\\
    $R_{out}\!(3,2,2) \leftarrow \naf R_{out}\!(3,3,3)$;\\
    $R_{out}\!(3,3,3) \leftarrow \naf R_{out}\!(3,2,2)$;
 \end{minipage}%
 \begin{minipage}{.3\textwidth}
    $R_{in}\!(1,1,1) \leftarrow \naf R_{out}\!(1,1,1)$;\\
    $R_{in}\!(1,2,1) \leftarrow \naf R_{out}\!(1,2,1)$;\\
    $R_{in}\!(3,2,2) \leftarrow \naf R_{out}\!(3,2,2)$;\\
    $R_{in}\!(3,3,3) \leftarrow \naf R_{out}\!(3,3,3)$;
 \end{minipage}%
 \begin{minipage}{.24\textwidth}
    $Q(1,1) \leftarrow R_{in}\!(1,1,1)$;\\
    $Q(1,1) \leftarrow R_{in}\!(1,2,1)$;\\
    $Q(3,2) \leftarrow R_{in}\!(3,2,2)$;\\
    $Q(3,3) \leftarrow R_{in}\!(3,3,3)$;
 \end{minipage}
 \begin{minipage}{.165\textwidth}
    $R_{in}\!(2,2,2) \leftarrow$;\\
    $R_{in}\!(2,2,3) \leftarrow$;\\
    $Q(2,2) \leftarrow$;\\
    $Q(2,3) \leftarrow$.
 \end{minipage}

\medskip
\noindent The program above has four stable models:
\begin{enumerate}
\item $I \cup \{R_{out}(1,1,1),$ $R_{out}(3,2,2),$ $R_{in}(1,2,1),$ $R_{in}(3,3,3),$ $Q(3,3)\}$;
\item $I \cup \{R_{out}(1,2,1),$ $R_{out}(3,2,2),$ $R_{in}(1,1,1),$ $R_{in}(3,3,3),$ $Q(3,3)\}$;
\item $I \cup \{R_{out}(1,1,1),$ $R_{out}(3,3,3),$ $R_{in}(1,2,1),$ $R_{in}(3,2,2),$ $Q(3,2)\}$;
\item $I \cup \{R_{out}(1,2,1),$ $R_{out}(3,3,3),$ $R_{in}(1,1,1),$ $R_{in}(3,2,2),$ $Q(3,2)\}$;
\end{enumerate}
where $I = \{R_{in}(2,2,2),$ $R_{in}(2,2,3),$ $Q(1,1), Q(2,2),$ $Q(2,3)\}$ is the set of cautious consequences.
\end{example}

\smallskip
\noindent {\em Stable model search.}
Given a normal ASP program $P$, its stable models can be computed by means of an algorithm similar to the DPLL backtracking search algorithm \cite{davi-etal-62} adopted by SAT solvers.
In this algorithm, atoms are associated with a truth value among true, false and undefined. 
Moreover, atoms are associated with a nonnegative integer called \emph{level}.
Actually, the backtracking search is usually preceded by \emph{simplification} techniques \cite{DBLP:conf/sat/EenB05}.
Simplifications include polynomial time algorithms that
(i) identify atoms whose truth value is deterministically implied by the input program, 
(ii) strengthen and remove rules, and 
(iii) eliminate atoms by means of rule rewriting.
Then, the nondeterministic search starts choosing \emph{branching atoms} according to some heuristic, propagating consequences of these choices until either a stable model is found, or a conflict is detected.
The level associated with an atom $a$ is the depth of the search tree in which $a$ has been either chosen or determined, where atoms assigned by (i) have level 0. 
The propagation step is polynomial, and corresponds to unit propagation in SAT solvers.
When a conflict is found, previous choices and their consequences are unrolled until consistency is restored (backjumping; \citeNP{gasc-79}).
Modern solvers analyze conflicts in order to learn constraints that are implicit in the original program and inhibit future explorations of the same (conflictual) branch of the search tree.
This learning step corresponds to clause learning in SAT \cite{zhan-etal-2001}, and is usually complemented with heuristic techniques that control the number of learned constraints, 
and possibly restart the computation in order to explore different branches of the search tree.
Restart policies are based on specific sequences of thresholds that guarantee termination of the algorithm \cite{gome-etal-98,luby-etal-93}.

\section{Computation of Cautious Consequences}\label{sec:algorithms}

\begin{figure}[t]
 \begin{minipage}{\textwidth}
    \begin{algorithm}[H]\small
     \SetKwInOut{Input}{Input}\SetKwInOut{Output}{Output}
     \Input{a program $P$ and a set of atoms $Q$}
     \Output{atoms in $Q$ that are cautious consequences of $P$, or $\bot$}

        $U := \emptyset;\quad O := Q;\quad L := \emptyset$\;
        $I$ := ComputeStableModel($P$, $L$, \A)\;
        \If{$I = \bot$}{
            \Return{$\bot$\;}
        }
        $O := O \cap I$\;
        \While{$U \neq O$}{
            \tcp{EnumerationOfModels or other procedure}
        }
        \Return{$U$}\;
     
     \caption{CautiousReasoning}\label{alg:cautious}
    \end{algorithm}
 \end{minipage}
 \begin{minipage}{.5\textwidth}
    \begin{procedure}[H]\small
        $P$ := $P \cup {}$ Constraint($I$)\;
        $I$ := ComputeStableModel($P$, $L$, \A)\;
        \uIf{$I = \bot$}{
            $U := O$\;
        }
        \Else{
            \phantom{\;}
            $O := O \cap I$\;
        }
     \caption{EnumerationOfModels() \hfill (A1)}\label{alg:enum}
    \end{procedure}
 \end{minipage}%
 \begin{minipage}{.5\textwidth}
    \begin{procedure}[H]\small
        $P$ := $P \cup {}$ Constraint($O$)\;
        $I$ := ComputeStableModel($P$, $L$, \A)\;
        \uIf{$I = \bot$}{
            $U := O$\;
        }
        \Else{
            $P$ := $P \setminus {}$ Constraint($O$)\;
            $O := O \cap I$\;
        }
     \caption{OverestimateReduction() \hfill (A2)}
    \end{procedure}
 \end{minipage}
 \begin{minipage}{.5\textwidth}
    \begin{procedure}[H]\small
        $a := $ OneOf($O \setminus U$)\;
        $I$ := ComputeStableModel($P$, $L$, $\{a\}$)\;
        \uIf{$I = \bot$}{
            $U := U \cup \{a\}$\;
        }
        \Else{
            $O := O \cap I$\;
        }
     \caption{IterativeCoherenceTesting() \hfill (A3)}
    \end{procedure}
 \end{minipage}%
 \begin{minipage}{.5\textwidth}
    \begin{procedure}[H]\small
        $a := $ OneOf($O \setminus U$)\;
        $I$ := ComputeUpToNextRestart($P$, $L$, $\{a\}$)\;
        \uIf{$I = \bot$}{
            $U := U \cup \{a\}$\;
        }
        \ElseIf{$I \neq RESTART$}{
            $O := O \cap I$\;
        }
     \caption{IterativePartialCoherTest()\phantom{g} \hfill (A4)}
    \end{procedure}
 \end{minipage}
 \begin{minipage}{\textwidth}
    \begin{function}[H]\small
     \SetKwInOut{Global}{Global variables}
     \Global{the underestimate $U$}
        \Repeat{$I \neq RESTART$}{
            $U := U \cup \{a \in Q \mid L \mbox{ contains } \bot \leftarrow \naf a \}$\;
            $I$ := ComputeUpToNextRestart($P$, $L$, OneOf($C$))\;
        }
        \Return{$I$}\;
     
     \caption{ComputeStableModel$^*$($P$: \textbf{program}, $L$: \textbf{learned constraints},  $C$: \textbf{set of atoms})}\label{proc:ubInc}
    \end{function}
 \end{minipage}
\end{figure}

Several strategies for computing cautious consequences of a given program are reported in this section.
Some of these strategies aim at solving the problem producing overestimates of the solution, which are improved and eventually result in the set of cautious consequences of the input program.
Among them are the algorithms implemented by the ASP solvers DLV \cite{alvi-etal-2011-dl2} and clasp \cite{gebs-etal-2012-aij}, respectively called \emph{enumeration of models} and \emph{overestimate reduction} in the following.
Other strategies can in addition produce sound answers during the computation of the complete solution, thus providing underestimates also when termination is not affordable in reasonable time.
One of these strategies is \emph{iterative coherence testing}, an adaptation of an algorithm computing backbones of propositional formulas \cite{DBLP:conf/ecai/Marques-SilvaJL10}.
To the best of our knowledge, no previous attempt to bring such an algorithm in ASP is reported in the literature.
A variant of this algorithm, namely \emph{iterative partial coherence testing}, is also introduced here.
Finally, a strategy for obtaining underestimates from enumeration of models and overestimate reduction is presented, which can also be used to improve the other algorithms.
More in detail, the algorithms considered here have a common skeleton, reported as Algorithm~\ref{alg:cautious}.
They receive as input a program $P$ and a set of atoms $Q$ representing answer candidates of a query, and produce as output either the largest subset of $Q$ that only contains cautious consequences of $P$, in case $P$ is coherent, or $\bot$ when $P$ is incoherent.
Initially, the underestimate $U$ and the overestimate $O$ are set to $\emptyset$ and $Q$, respectively (line~1).
A coherence test of $P$ is then performed (lines~2--4) by calling function ComputeStableModel, which actually implements stable model search as described in Section~\ref{sec:pre}.
(To simplify the presentation, branching atoms are assumed to be assigned the false truth value.)
The first argument of the function is a program $P$.
The second argument is a set of learned constraints, which is initially empty.
The third argument is a set $C$ of atoms used to restrict branching atoms of level 1. 
The function returns either $I$ in case a stable model $I$ of $P$ is found, or $\bot$ otherwise.
Note that $\bot$ is returned not only when $P$ is incoherent, but in general when each stable model $M$ of $P$ is such that $C \subseteq M$.
Similarly, when $I$ is returned, stable model $I$ satisfies $C \not\subseteq I$.
When $C = \A$, the condition $C \not\subseteq I$ is trivially satisfied because $I \subseteq \A \setminus \{\bot\}$ by definition of interpretation.
When $C = \{a\}$ for some atom $a$, instead, this function results in an incremental stable model search in which $a$ is forced to be false \cite{DBLP:journals/entcs/EenS03}.

The first stable model found improves the overestimate (line~5).
At this point, estimates are improved according to different strategies until they are equal (line~6).
EnumerationOfModels adds to $P$ a constraint that eliminates the last stable model found (line~1).
In fact, function Constraint($\{a_1, \ldots, a_n\}$) returns a singleton of the form $\{\bot \leftarrow a_1, \ldots, a_n\}$.
The algorithm then searches for a new stable model (line~2) to improve the overestimate (line~6).
If no new stable model exists, the underestimate is set equal to the overestimate (lines~3--4), thus terminating the computation.
OverestimateReduction is similar, but the constraint added is obtained from the current overestimate (line~1).
In this way, when a new stable model is found, an improvement of the overestimate is guaranteed, and the constraint can be reduced accordingly (lines~6 and 1).

The strategy implemented by IterativeCoherenceTesting can also improve the underestimate many times during its computation.
In fact, one cautious consequence candidate is selected by calling function OneOf (line~1).
This candidate is then constrained to be false and a stable model is searched (line~2).
If none is found then the underestimate can be increased (lines~3--4).
Otherwise, the overestimate can be improved (lines~5--6).
IterativePartialCoherenceTesting is similar, but forces falsity of a candidate only up to the next restart (lines~1--2).
In fact, ComputeStableModel is replaced by ComputeUpToNextRestart, a function that searches for a stable model but also terminates when a restart occurs, in which case it returns the value $RESTART$.
In this way, the algorithm can select the most promising candidate after each restart.

Variants of these four algorithms can be obtained by replacing function ComputeStableModel with function ComputeStableModel$^*$, which actually implements stable model search, but also improves the current underestimate after each restart (line~2).

\begin{theorem}\label{thm:main}
Let $P$ be a program and $Q \subseteq \A$ a set of atoms.
CautiousReasoning($P$,$Q$) terminates after finitely many steps and returns $Q \cap CC(P)$ if $P$ is coherent;
otherwise, it returns $\bot$.
Moreover, $U \subseteq Q \cap CC(P) \subseteq O$ holds at each step of computation.
The claim holds for all variants of Algorithm~1.
\end{theorem}

\section{Implementation and Experiment}

We implemented the algorithms introduced in the previous section in order to analyze their performances.
Details on the implementation, on the tested benchmarks, and on the obtained results are reported in this section.

\subsection{Implementation}
Algorithms described in Section~\ref{sec:algorithms} are implemented in an experimental branch of the ASP solver WASP \cite{DBLP:conf/lpnmr/AlvianoDFLR13}, distributed under the Apache 2.0 license.
Source codes can be downloaded from the branch \emph{queries} of the public GIT repository \emph{https://github.com/alviano/wasp.git}.

WASP implements ASP solving with backjumping \cite{gasc-79}, learning \cite{zhan-etal-2001} and restarts \cite{gome-etal-98}. 
More in detail, the branch of WASP used in this experiment implements support propagation via program completion \cite{lier-mara-2004-lpnmr}, branching heuristics and deletion strategy inspired by MiniSAT \cite{een-etal-2003}, and simplifications via subsumption and atom elimination techniques as described by \citeN{DBLP:conf/sat/EenB05}.
(Actually, atom elimination, called variable elimination in SAT, is not applied on atoms involved in queries.)
In the following, $A2$, $A3$, $A4$ will denote WASP running Algorithm~1 with procedures OverestimateReduction, IterativeCoherenceTesting, and IterativePartialCoherenceTesting, respectively.
$A2^{*}$, $A3^{*}$, $A4^{*}$ will instead denote the variants using procedure ComputeStableModel$^*$.
Procedure EnumerationOfModels is not considered in the analysis since it is significantly outperformed by the other strategies in general.

We also implemented a proof-of-concept prototype of a parallel system, in the following referred to as $multi$.
It consists of a master controller implemented in Python that coordinates the execution and the exchange of information of two instances of WASP.
In particular, estimates and learned constraints of size at most two are exchanged.
In our experiment, $multi$ runs $A2^{*}$ and $A4^{*}$, but $A4^*$ in this case does not perform the first coherence check (lines~2--4 of Algorithm~1) in order to avoid a redundant computation. Other combinations of algorithms are possible, but not considered in our analysis. 
Results for {\em multi} average real time over three runs. 

\subsection{Benchmark settings}

We compared the implemented algorithms on three benchmarks,
corresponding to different applications of cautious reasoning,
briefly described below. 

\smallskip
\noindent
\emph{Multi-Context Systems Querying (MCS).}
Multi-context systems \cite{DBLP:conf/aaai/BrewkaE07} are a formalism for interlinking heterogeneous knowledge bases, 
called contexts, using bridge rules that model the flow of information among contexts. 
Testcases in this benchmark are roughly those of the third ASP competition \cite{DBLP:journals/tplp/CalimeriIR14}, where each context is modeled by a normal logic program under the stable model semantics. 
We actually made the testcases harder by requiring the computation of all pairs of the form $\tuple{c,a}$ such that atom $a$ is true in context $c$, while in the original testcases a single pair of that form was involved in the query.
The benchmark contains 53 of the 73 instances submitted to the third ASP competition.
We in fact excluded instances corresponding to incoherent theories, which are solved by the first coherence test in around 6 seconds on the average, and always in less than 14 seconds.

\smallskip
\noindent
\emph{Consistent Query Answering (CQA) }
is a well-known application of ASP \cite{aren-etal-2003-tplp,DBLP:journals/tplp/MannaRT13} described in Section~\ref{sec:introduction}.
We considered the benchmark proposed by \citeN{DBLP:journals/pvldb/KolaitisPT13},
and in particular query {\em Q3} encoded according to the rewritings by \citeANP{DBLP:journals/tplp/MannaRT13}.
The benchmark contains 13 randomly-generated databases of increasing size ranging from 1000 to 7000 tuples per relation.
Each relation contains around 30\% of primary key violations.

\smallskip
\noindent
\emph{SAT Backbones (SBB).}
The backbone of a propositional formula $\varphi$ is the set of literals that are true in all models of $\varphi$.
When $\varphi$ is a set of clauses over variables $v_1,\ldots,v_n$ ($n \geq 1$), satisfiability of $\varphi$ can be modeled in ASP by rules
$t_i \leftarrow \naf f_i$ and
$f_i \leftarrow \naf t_i$
($i = 1, \ldots, n$),
and introducing a constraint for each clause in $\varphi$.
Backbone computation thus corresponds to the computation of cautious consequences of an ASP program.
The benchmark contains 20 industrial instances used in the SAT Challenge 2012 \cite{DBLP:journals/aim/JarvisaloBRS12}.

The experiment was run on a Mac Pro equipped with two 3 GHz Intel Xeon X5365 (quad core) processors, 
with 4 MB of L2 cache and 16 GB of RAM, running Debian Linux 7.3 (kernel ver. 3.2.0-4-amd64). 
Binaries were generated with the GNU C++ compiler 4.7.3-4 shipped by Debian.
The parallel controller was interpreted by Python 3.3.2.
Time and memory limits were set to 600 seconds and 8 GB, respectively.
Performance was measured using the tool RunLim ({\em http://fmv.jku.at/runlim/}).
All instances were grounded by gringo 3.0.5 \cite{DBLP:conf/lpnmr/GebserKKS11}, whose execution time is not included in our analysis because our focus is on propositional programs.
We however report that the grounding time was often less than 1 second, with a peak of around 5 seconds for the largest 10 instances of MCS.

\subsection{Discussion of the results}
The performance of the algorithms for computing cautious consequences introduced in Section~\ref{sec:algorithms} can be studied from several perspectives. 
On the one hand, we want to know which solution performs better and in which cases.
On the other hand, we are interested in analyzing the rate at which each algorithm produces sound answers.

\begin{table}[b]
	\caption{Average running time and number of solved instances} \vspace{0.1cm}
	\label{tab:results}
	\begin{tabular*}{0.89\columnwidth}{lcc|ccc|ccc|ccc|ccc}
	\multicolumn{3}{c}{} 		 			&\multicolumn{3}{|c}{\bf{A2*}} 			&\multicolumn{3}{|c}{\bf{A3*}} 				&\multicolumn{3}{|c}{\bf{A4*}}				& \multicolumn{3}{|c}{\bf{multi}}				\\
	 			& \bf{\#} 		& \bf{\#$_{all}$} 	& \bf{sol.} 	& \bf{t} 			&\bf{t$_{all}$} 	&\bf{sol.} 		& \bf{t}			& \bf{t$_{all}$} 	&\bf{sol.} 		& \bf{t}	 		& \bf{t$_{all}$} 	& \bf{sol.} 		& \bf{t} 	 		& \bf{t$_{all}$} 	\\
\cline{1-15}
\bf{MCS}			& 53 		& 23 		& 23 	& 181.9 		& 181.9 		& 39 		& 254.1 		& 102.6 		& 40 		& 261.6 		& 102.4 		& 40 		& 177.9 		& \phantom{0}69.7 		\\
\bf{CQA} 			& 13 		& 12 		& 12	 	& 118.8 		& 118.8 		& 13 		& \phantom{0}89.5 		& \phantom{0}62.1 		& 13 		& \phantom{0}89.2 		& \phantom{0}61.8 		& 13 		& 137.2 		& \phantom{0}51.0 		\\
\bf{SBB}			& 20 		& 14 		& 15 	& \phantom{0}53.4 		& \phantom{0}41.9 		& 14 		& \phantom{0}65.7 		& \phantom{0}65.7 		& 14 		& \phantom{0}52.1 		& \phantom{0}52.1 		& 17 		& \phantom{0}98.8 		& \phantom{0}51.2 		\\
\cline{1-15}
\bf{Total} 		& \bf{86} 	& \bf{49} 		& \bf{50} 	& \bf{118.0}	& \bf{114.2} 	& \bf{66} 	& \bf{136.4}	& \bf{\phantom{0}76.8}	& \bf{67}	& \bf{134.3}	& \bf{\phantom{0}72.1}	& \bf{70}	& \bf{138.0}	& \bf{\phantom{0}57.3}	\\
    \end{tabular*}
\end{table}

\smallskip 
\noindent {\em Overall performance.}
Table~\ref{tab:results} summarizes the number of solved instances and average running times. 
In particular, the first two columns report the total number of instances (\#) and the number of instances that are solved by all solvers (\#$_{all}$), respectively; the remaining columns report the number of solved instances within the time-out (sol.), the average running times on solved instances (t) and on instances solved by all algorithms (t$_{all}$). 
Comparing the single-process approaches, we note that \Any{4} solves more instances
and is also the fastest on average in the instances solved by all algorithms, even if the performance of \Any{3} is comparable.
\Any{4} and \Any{3} outperform \Any{2} in MCS, and are faster also in CQA.
On the other hand, \Any{2} performs well in SBB, solving one instance more than 
\Any{3} and \Any{4}, and being faster on the average. 
Note that, as expected, if one considers both the number of solved instances and running time, 
\Alg{2}, \Alg{3}, and \Alg{4} perform as \Any{2}, \Any{3}, and \Any{4},  respectively.
Concerning {\em multi}, it provides in general the best performance. 

\begin{figure}
 \begin{minipage}{\textwidth}
  \centering\includegraphics[trim = 0mm 14mm 0mm 0mm, clip, width=.9\textwidth]{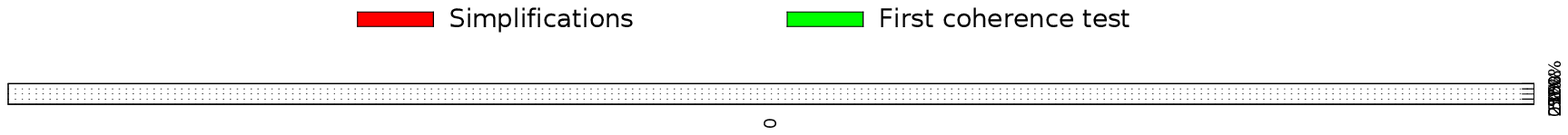}
 \end{minipage}
 \subfigure[Sound answers]{\label{fig:simp-sound}
  \includegraphics[trim = 4mm 0mm 8mm 0mm, clip, angle=-90, width=.5\textwidth]{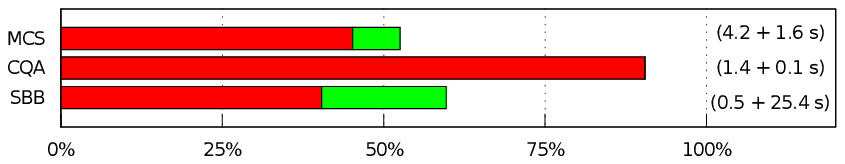}
 }%
 \subfigure[Candidates reduction]{\label{fig:simp-cand}
  \includegraphics[trim = 4mm 0mm 8mm 0mm, clip, angle=-90, width=.5\textwidth]{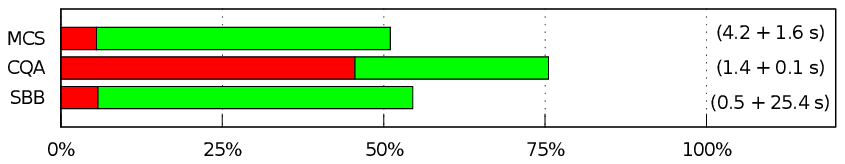}
 }
 \caption{Sound answers and candidates reduction from simplifications and from first coherence test}\label{fig:bar-simp}
\end{figure}

\smallskip 
\noindent {\em Detailed analysis.}
An important feature of the algorithms analyzed in this paper is the ability to produce both sound answers and overestimates during the computation. 
%
Figure~\ref{fig:bar-simp} reports, for each benchmark, the average percentage of (a) sound answers produced and (b) candidates reduction within the initial steps of the computation.
In particular, we plot the effects of simplifications and of the first coherence test.
The improvement of the overestimate reported in Figure~\ref{fig:simp-cand} is significant. The first steps of the computation are able to reduce the number of candidates of at least 51\% (in MCS) up to around 75\% (in CQA). Simplifications are already very effective in CQA, where candidates are reduced of around 45\%. 
%
It is important to note that the reduction of candidates at this stage applies to all algorithms, while sound answers are produced only by anytime algorithms. 
This is effective in practice, as shown in Figure~\ref{fig:simp-sound}.
Indeed, anytime algorithms print from 40\% (in SBB) to 90\% (in CQA) of sound answers already after simplifications, which requires few seconds on the average. 
The first coherence test further improves the underestimate, which ranges from 52\% (in MCS) to around 91\% (in CQA).
However, we observe that the first coherence test may require some time (25s on the average for SBB instances, with a peak of 193s), which motivated the starred variants.
In fact, starred variants can produce underestimates at each restart, not only when a coherence test is completed. 
Actually, \Any{2}, \Any{3} and \Any{4} improve progressively the underestimate up to around an additional 30\% before the first stable model is found, which is desirable on hard instances.

\begin{figure}
 \centering
 \subfigure[Basic algorithms on instance 34 of MCS]
   {\label{fig:mcsScissor}\includegraphics[width=4.2cm,angle=-90]{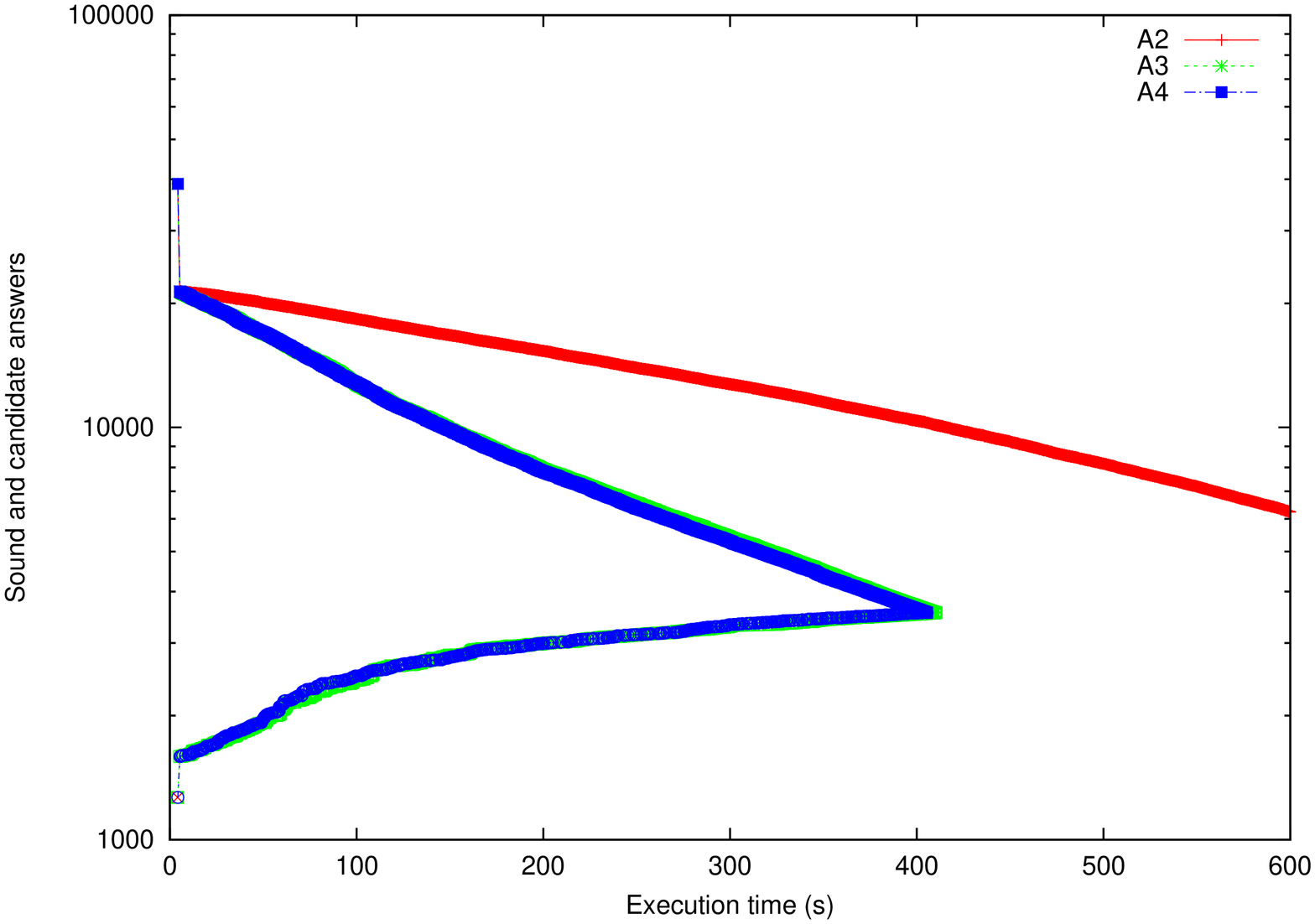}}
 \hspace{5mm}
 \subfigure[Starred algorithms on instance 34 of MCS]
   {\label{fig:mcsScissorAny}\includegraphics[width=4.2cm,angle=-90]{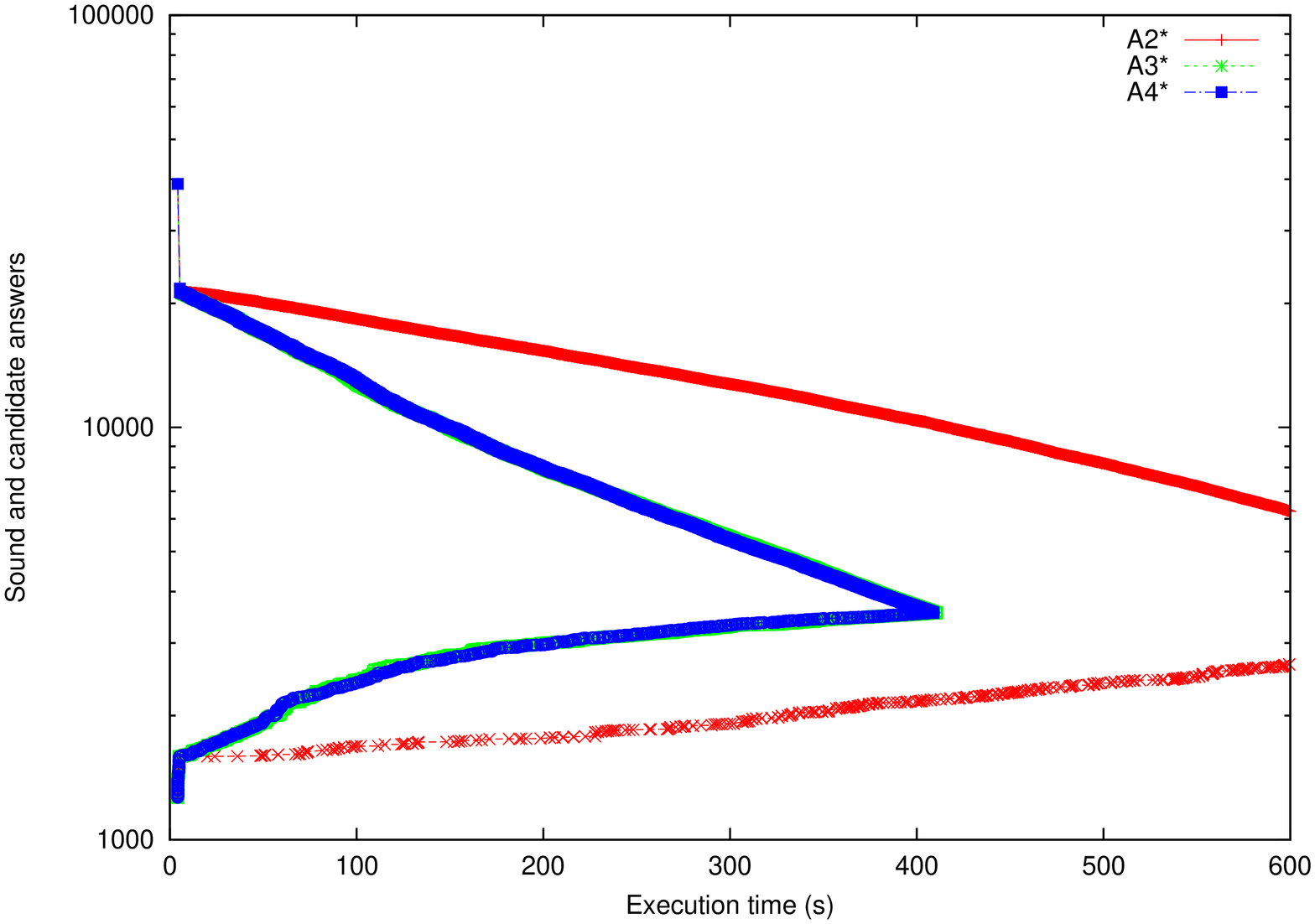}}
 \subfigure[Basic algorithms on mrpp\_6x6\#12\_16 of SBB]
   {\label{fig:satScissor}\includegraphics[width=4.1cm,angle=-90]{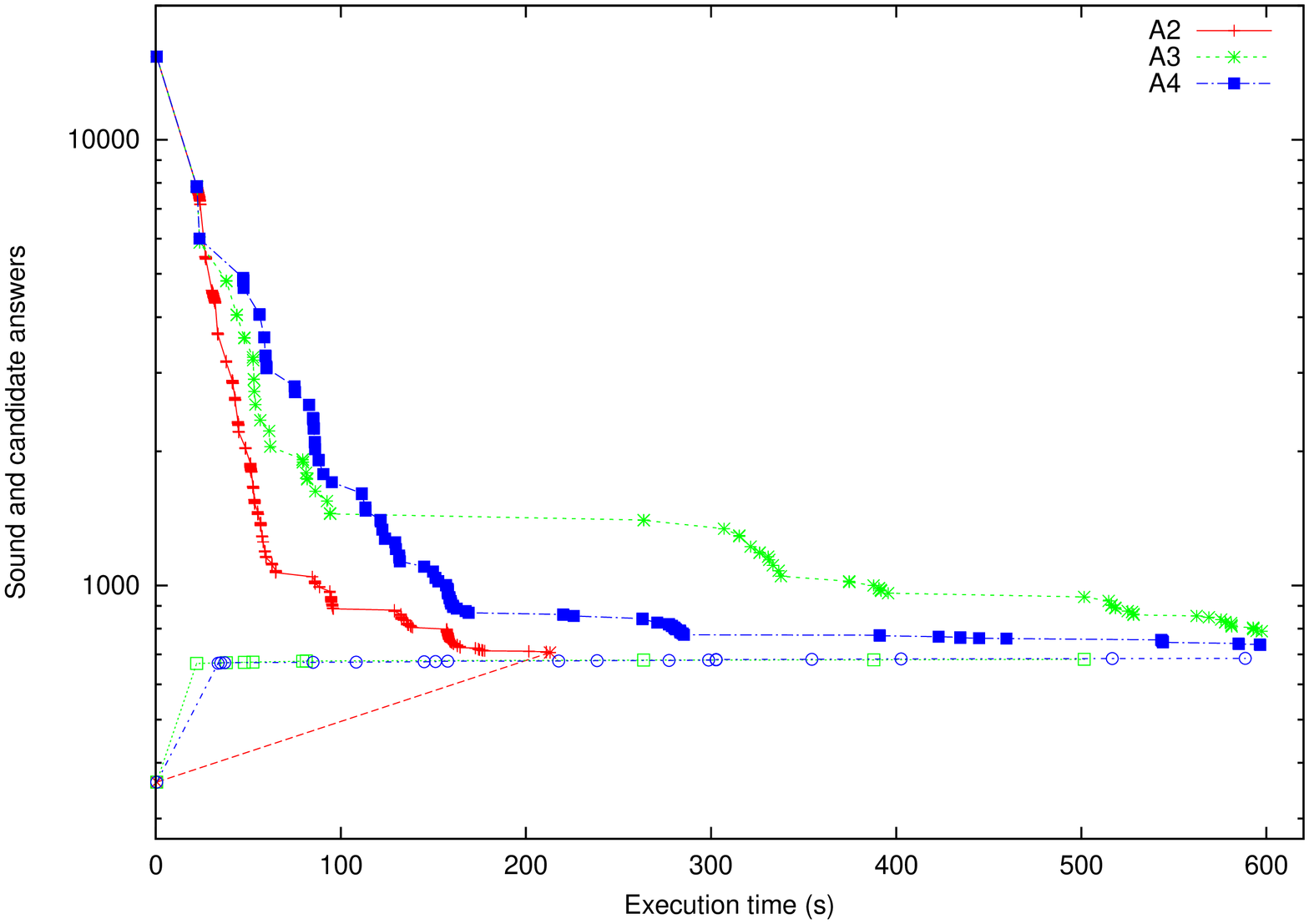}}
 \hspace{5mm}
 \subfigure[Starred algorithms on mrpp\_6x6\#12\_16 of SBB]
   {\label{fig:satScissorAny}\includegraphics[width=4.1cm,angle=-90]{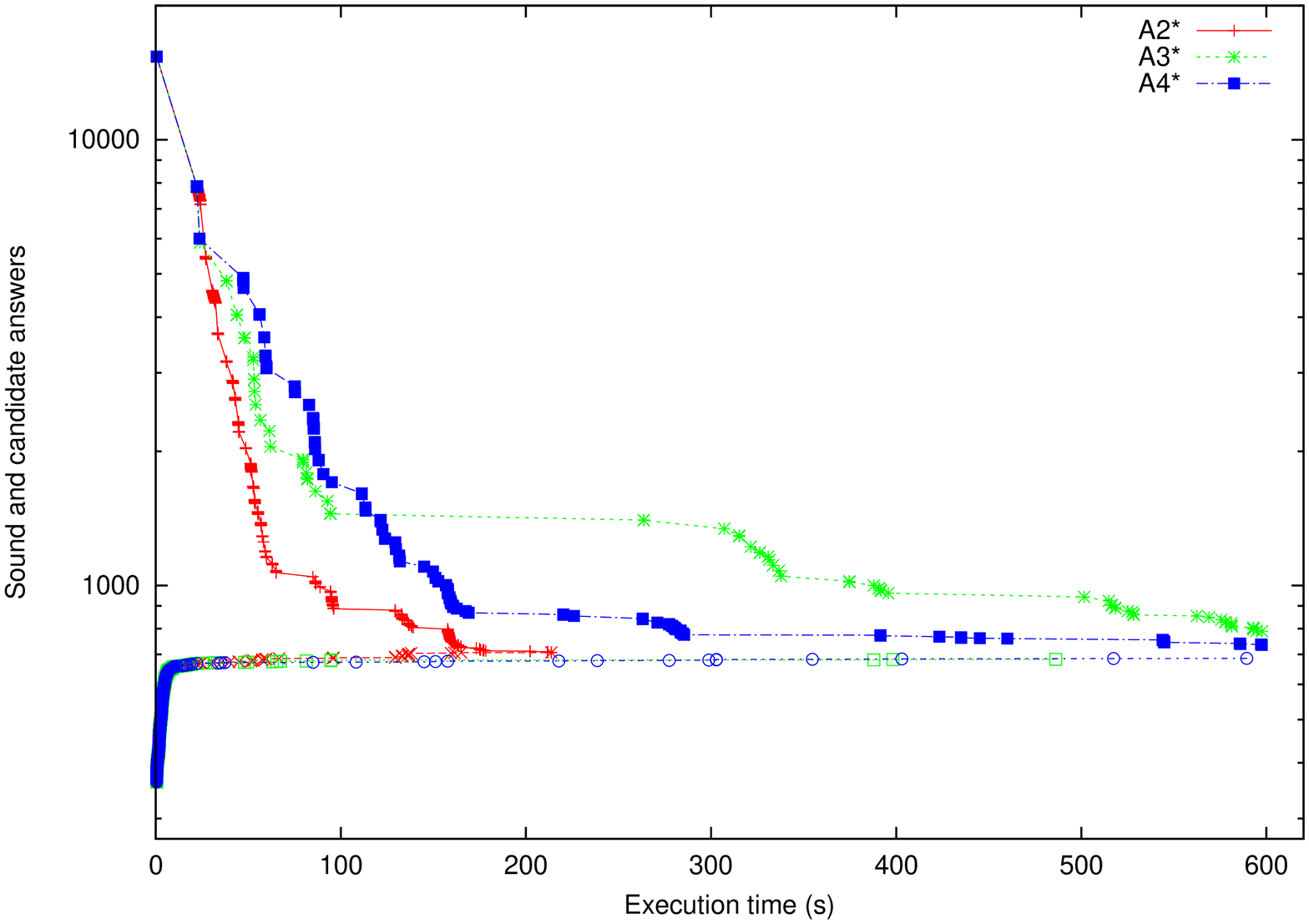}}
 \hspace{5mm}
 \subfigure[A2$^*$, A4$^*$ and multi on mrpp\_6x6\#12\_16 of SBB]
   {\label{fig:multiScissor1}\includegraphics[width=4.1cm,angle=-90]{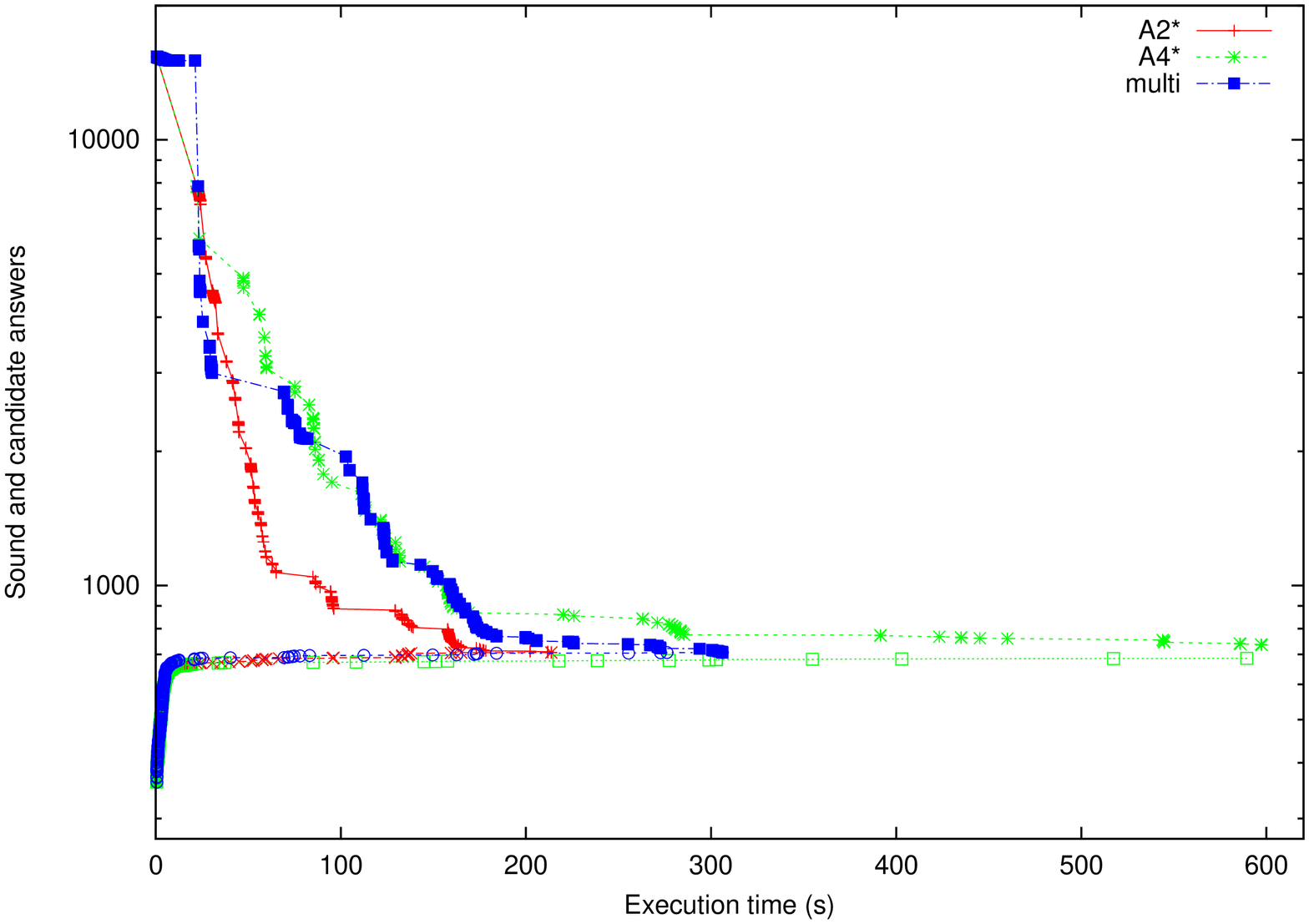}}
 \hspace{5mm}
 \subfigure[A2$^*$, A4$^*$ and multi on mrpp\_4x4\#10\_16 of SBB]
   {\label{fig:multiScissor2}\includegraphics[width=4.1cm,angle=-90]{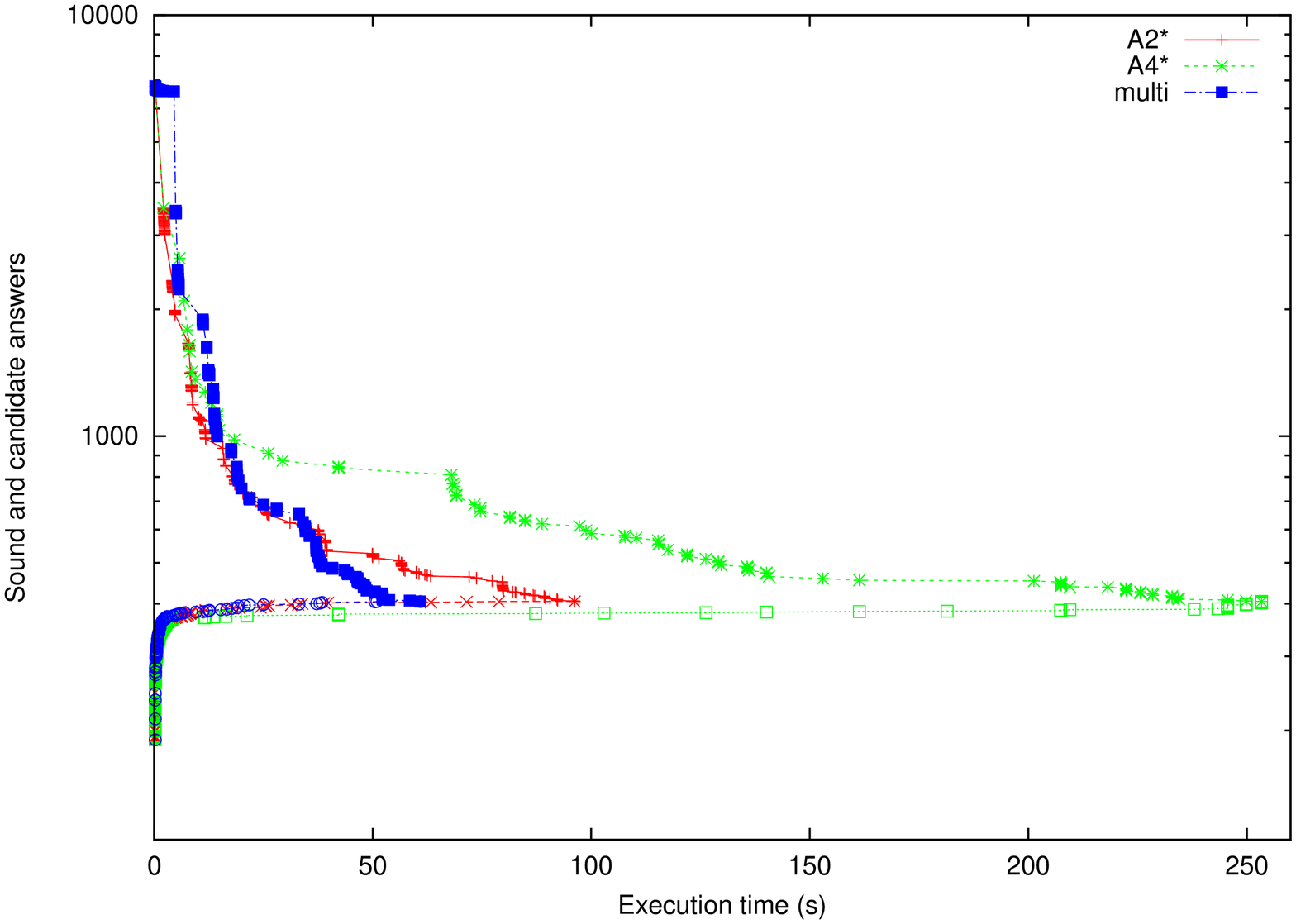}}
 \caption{Overestimate and underestimate improvement during execution}\label{fig:scissor}
 \end{figure}

In order to further confirm the above observations, we analyze in detail the behavior of the algorithms after simplifications.
In particular, Figure~\ref{fig:scissor} plots both the number of sound answers (line below) 
and the number of candidate answers (line above) over time.
In particular, Figure~\ref{fig:mcsScissorAny} is devoted to the starred algorithms on an instance of MCS,
whereas Figure~\ref{fig:mcsScissor} plots the behavior of the basic algorithms on the same instance.
First, we note that \Alg{3} and \Alg{4} perform similarly and outperform \Alg{2}, which timed out. 
Notably, \Any{2} can produce the underestimate (see the bottom line in Figure~\ref{fig:mcsScissorAny}) whereas \Alg{2} can only
print the overestimate (there is no underestimate line for \Alg{2} in Figure~\ref{fig:mcsScissor}). 
In general, \Alg{3} and \Alg{4} are able to improve their estimates better than \Alg{2}.
Note that there is a point in the plots for each improvement of estimates, and lines are very dense on MCS instances. 
This confirms that MCS instances have a huge number of stable models that can be rapidly computed.
We observed an analogous behavior for CQA.
Plots for SBB instances on Figure~\ref{fig:satScissor} and Figure~\ref{fig:satScissorAny} have, instead, sparse lines, confirming that stable model search is harder for this benchmark.
Nevertheless, the starred algorithms can rapidly produce most of the sound answers. 
A deeper look at Figure~\ref{fig:satScissor} suggests that \Alg{2} is much faster than both \Alg{3} and \Alg{4} in solving this instance.
In fact, it improves the overestimate faster than any other algorithm.
%
Figure~\ref{fig:multiScissor1} and Figure~\ref{fig:multiScissor2} focus on {\em multi} and its components.
In Figure~\ref{fig:multiScissor2}, {\em multi} is faster and improves estimates better than \Any{2} and \Any{4}.
In contrast, {\em multi} is slower than \Any{2} in Figure~\ref{fig:multiScissor1}. 
A possible cause is the information exchange among processes that modifies the program handled by \Any{2} 
with constraints produced by \Any{3}, which in this case results in an harder instance also for the \Any{2} process. 
Note that {\em multi} has a non-deterministic behavior due to its parallel nature. 
Indeed, information exchanged may be different in different runs. 

\begin{figure}[t]
 \begin{minipage}{\textwidth}
  \centering\includegraphics[trim = 0mm 14mm 0mm 0mm, clip, width=.9\textwidth]{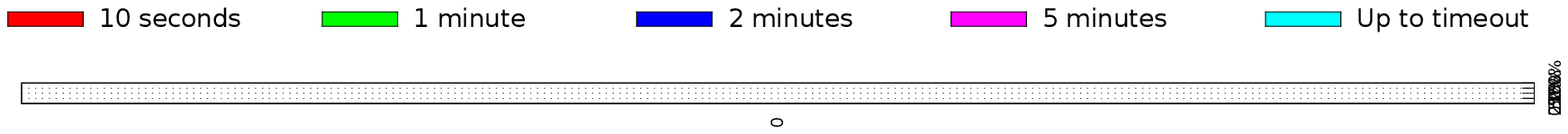}
 \end{minipage}
 \subfigure[Sound answers]{\label{fig:bar-sound}
  \includegraphics[trim = 4mm 0mm 8mm 0mm, clip, angle=-90, width=.5\textwidth]{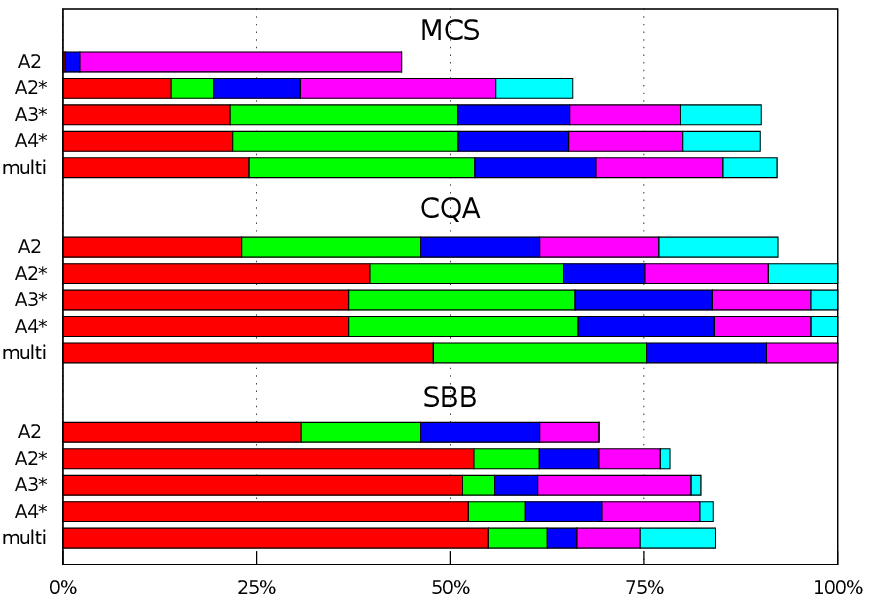}
 }%
 \subfigure[Candidates reduction]{\label{fig:bar-cand}
  \includegraphics[trim = 4mm 0mm 8mm 0mm, clip, angle=-90, width=.5\textwidth]{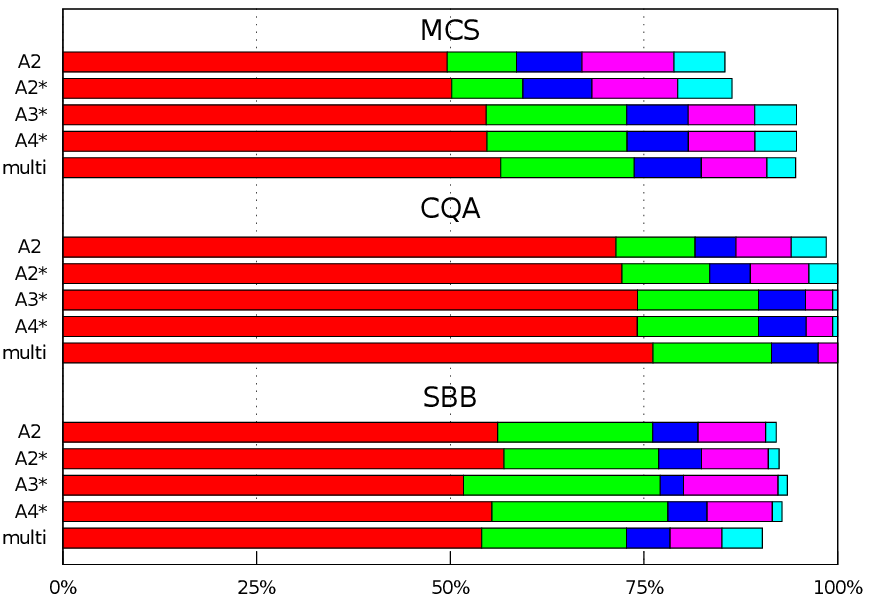}
 }
 \caption{Sound answers and candidates reduction after simplification}\label{fig:bar}
\end{figure}

More insights on the general behavior of the algorithms in the non-deterministic part of the computation can be obtained by looking at Figure~\ref{fig:bar}.
In particular, Figure~\ref{fig:bar-sound} reports the average percentage of sound answers produced  {\em after the simplification step}, while candidates reduction is shown on Figure~\ref{fig:bar-cand}.
\Alg{3} is not shown in the figure because it performs similarly to \Any{3} in this perspective.
The same holds for \Alg{4} and \Any{4}.
We point out that all bars refer to sound answers and candidates remaining after simplifications, also for \Alg{2}.
As a general observation, \Alg{2} prints sound answers only at the end of the computation, while other algorithms are anytime.
Consequently, \Alg{2} does not provide sound answers as soon as the other algorithms, as shown on Figure~\ref{fig:bar-sound}. 
Basically, \Alg{2} can print something in the first 10s only for easy instances, while
\Any{2} improves a lot in this respect. For example, \Any{2} outputs around 14\% of sound answers already in the first 10s of computation in MCS, while \Alg{2} produces no output. Nonetheless, \Any{3} and \Any{4} perform generally better than \Any{2}. The difference between \Any{3} and \Any{4} emerges only in SBB, where finding stable models is harder. In particular, by looking at Figure~\ref{fig:bar-sound}, \Any{4} produces more sound answers than \Any{3}, whereas \Any{3} is more effective in reducing the number of candidates on Figure~\ref{fig:bar-cand}. Note that \Any{4} may change the candidate to test at each restart, and in our implementation it selects the one with the largest value of activity (which very roughly means the one that was involved more often in conflicts). On the other hand, \Any{3} insists on the same candidate until the end of a stable model search. As a consequence, \Any{4} has more chances to find inconsistent branches and, therefore, to improve the underestimate.
On the contrary, \Any{3} has more chances to find a stable model and, thus, to improve the overestimate. 
This suggests that \Any{4} should be preferred, since it outputs sound answers more frequently, and also because, 
as discussed above, \Any{4} is faster than \Any{3} on the average (see Table~\ref{tab:results}).
Concerning \emph{multi}, we observe that it outperforms the alternatives on CQA and SBB,
while it is less effective in printing sound answers on the average for SBB instances.
This is due to the presence of an outlier in this benchmark that spoils the average.
That outlier can be explained by the already discussed non-deterministic behavior of {\em multi}.
Analogous considerations can be done for the reduction of candidates on Figure~\ref{fig:bar-cand}.

\begin{figure}
 \centering
 \subfigure[25\% of sound answers]
   {\label{fig:mcs25}\includegraphics[width=4.2cm,angle=-90]{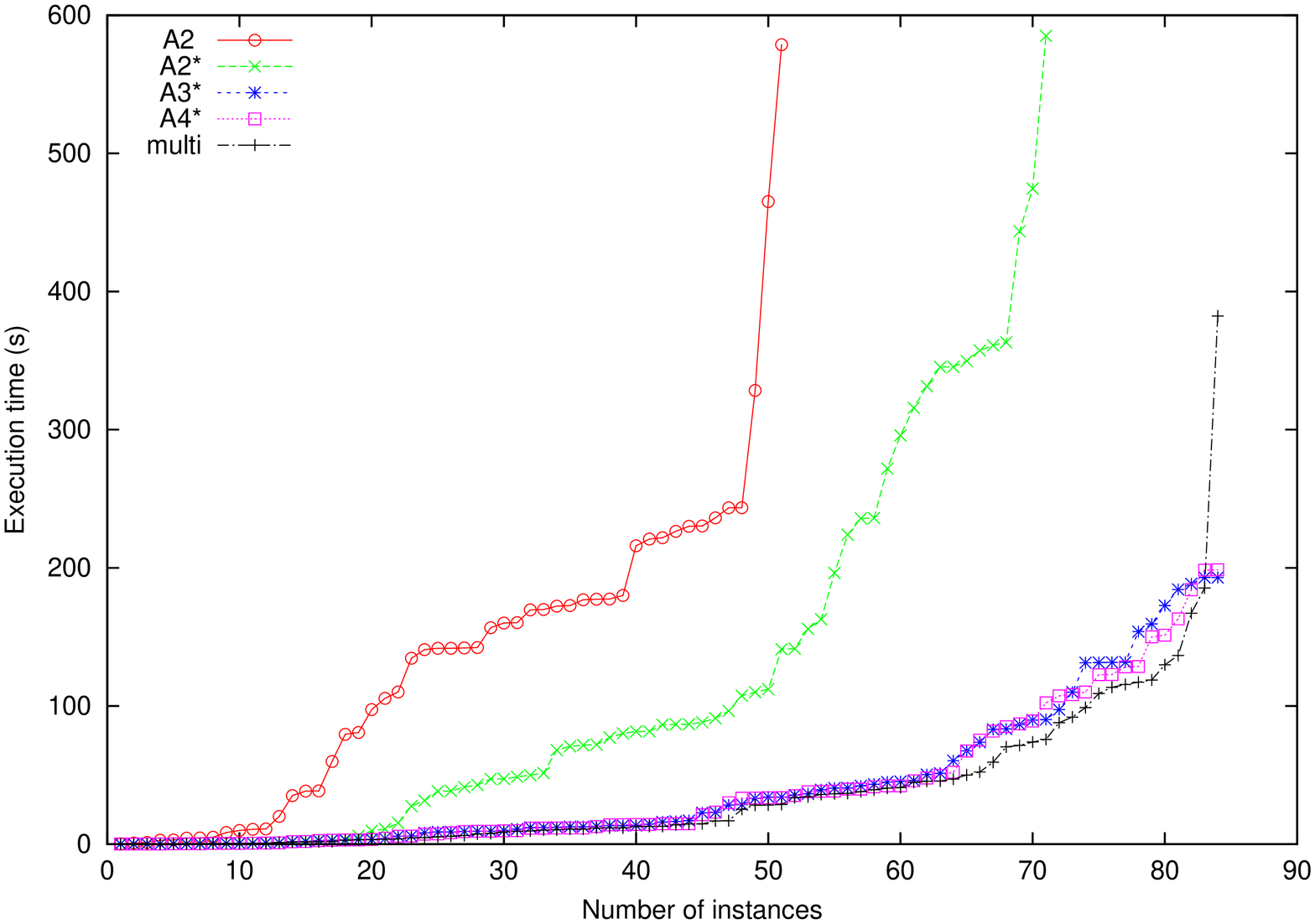}}
 \hspace{5mm}
 \subfigure[100\% of sound answers]
   {\label{fig:mcs100}\includegraphics[width=4.2cm,angle=-90]{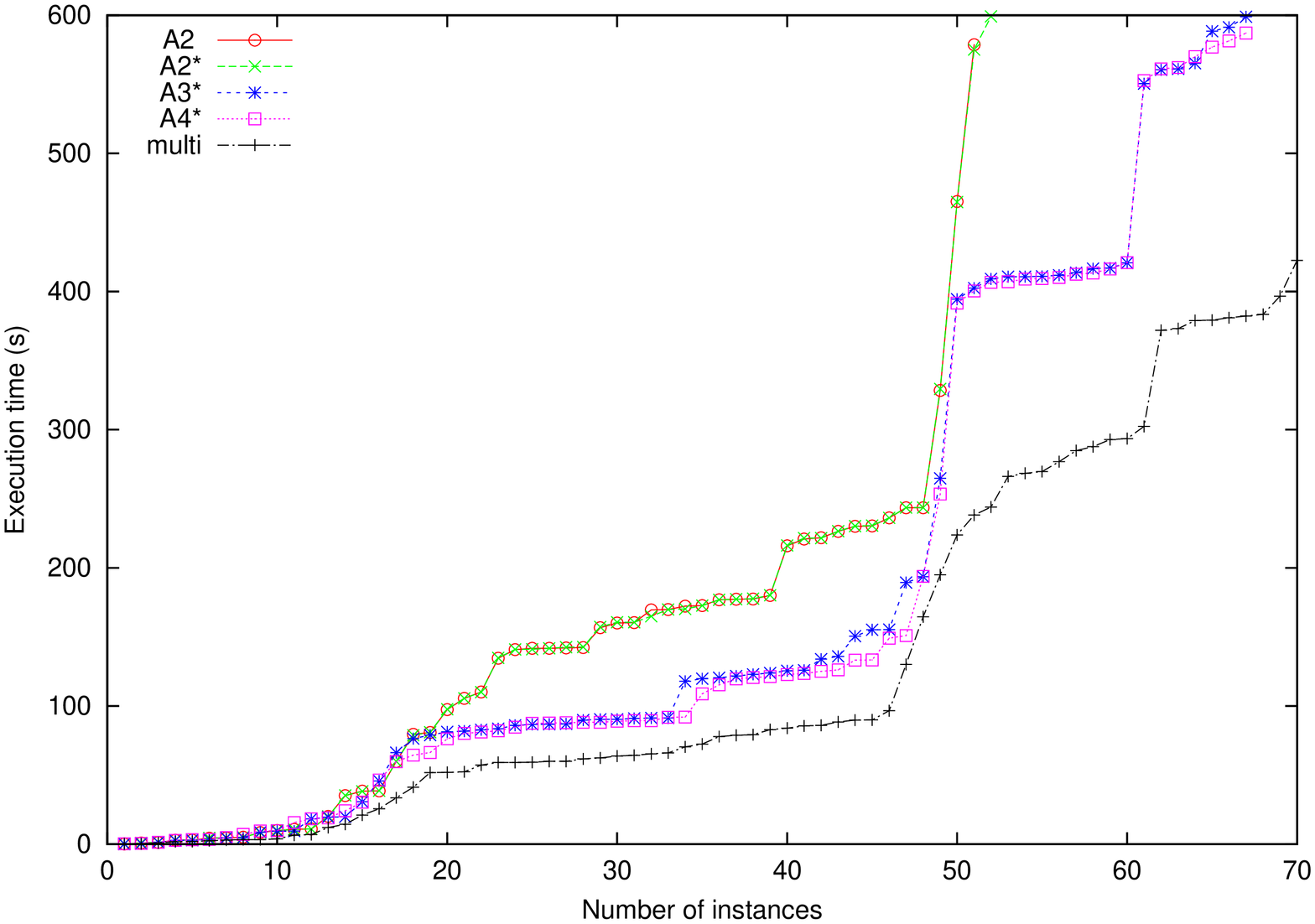}}
 \caption{Time performance of algorithms for computing underestimates}\label{fig:cactus}
 \end{figure}

Another perspective on the behavior of the various algorithms can be obtained by looking at Figure~\ref{fig:cactus}.
Here are reported, for each benchmark, two variants of the classical cactus plot. 
Recall that, in a cactus plot, the x-axis reports the number of instances that are solved within the time reported on the y-axis. 
Here we consider variants where
on the x-axis is the number of instances for which an algorithm printed 25\% (resp. 100\%) of sound answers
within the time reported on the y-axis.
We point out that anytime solvers can print 100\% of sound answers before the timeout, even if termination is not reached within the allotted time.
Figure~\ref{fig:cactus} confirms that \Alg{2} is slower than anytime algorithms in printing sound answers.
It is interesting to note that the anytime \Any{2} improves sensibly \Alg{2} in all benchmarks, especially at the beginning of the computation. 
We also observe that \Any{4} is slightly preferable to \Any{3}, and {\em multi} is the fastest solution. 
Note that the differences are more evident in the plots on the left that focus on the first 25\% of sound answers. 
Finally, we confirm that \Any{2} is the fastest single-process implementation in SBB.
Nonetheless, \Any{3} and \Any{4} print more sound answers also in non-terminating instances.

To sum up, anytime algorithms are convenient in practice because they determine several sound answers already in the first steps of computation.
In particular, the starred variants are preferable (\Any{4} leading the group) because they provide sound answers as soon as possible and are thus effective also in non-terminating instances.
%
Finally, the parallel system combining \Any{4} with \Any{2} is the best variant overall.

\section{Related Work}

The computation of cautious consequences in ASP is a feature available in two solvers, namely DLV \cite{MarateaRFL08} and clasp \cite{gebs-etal-2012-aij}.
The algorithm implemented by DLV is enumeration of models, while clasp implements overestimate reduction.
Our implementation differs from these solvers especially with respect to the output produced during the computation of cautious consequences.
In fact, DLV does not print any form of estimation during the computation, and clasp only prints overestimates.
Our implementation, instead, is anytime and thus prints both underestimates and overestimates during the computation.
Underestimates provide sound answers also when termination is not affordable in reasonable time, and are thus of practical importance for hard problems.
It is interesting to observe that among the strategies supported by our implementation there is \Any{2}, an anytime variant of the algorithm used by clasp that performed very well on two of our three benchmarks.
We also note that DLV and clasp feature brave reasoning, which is not currently supported by our implementation. 

Clasp, being a parallel ASP solver \cite{DBLP:journals/tplp/GebserKS12}, also supports the parallel computation of cautious consequences by means of the overestimate reduction algorithm.
Our proposal is different, as it is based on the combination of two different algorithms, namely iterative partial coherence testing and overestimate reduction, for reducing both estimates at the same time. 
However, we observe that our parallel implementation is a proof-of-concept prototype obtained by combining two instances of WASP (properly modified to share short learned constraints and answer estimates) controlled by a Python script.
It is devised to show the benefits of combining two different algorithms, while a more efficient implementation is subject of future work.

The computation of cautious consequences of a ground program is related to the problem of backbone computations of propositional formulas \cite{DBLP:conf/ecai/Marques-SilvaJL10,DBLP:conf/ijcai/SlaneyW01}.
In fact, the backbone of a propositional formula $\varphi$ is the set of literals that are true in all  models of $\varphi$.
Several algorithms for computing backbones of propositional formulas are based on variants of the iterative consistency testing algorithm \cite{DBLP:conf/ecai/Marques-SilvaJL10,janota-etal-2013-aicom}, which essentially corresponds to the iterative coherence testing algorithm analyzed in this paper.
Backbone search algorithms usually feature additional techniques for removing candidates to be tested, such as {\em implicant reduction} and {\em core-based chunking} \cite{DBLP:conf/tacas/RaviS04}.
Most of the implicant reduction techniques are not applicable to normal ASP programs because of the intrinsic minimality of stable models.
For example, backbone search algorithms can reduce their overestimate by removing all unassigned variables when a (partial) model is found; in our setting, ASP solvers always terminate with a complete assignment.
Core-based chunking, instead, requires a portfolio of algorithms \cite{janota-etal-2013-aicom} in order to be effective, which is beyond the scope of this paper.

Note that all considered algorithms work on ground programs. Combinations 
with query optimization techniques such as magic sets \cite{grec-2003,alvi-fabe-2011-aicomm} are possible but not the focus of the paper.
%
%
%

\section{Conclusion}
Several algorithms for computing cautious consequences of ASP programs were analyzed in this paper. 
At the time of this writing, ASP solvers do not implement anytime algorithms, which means that computation must terminate in order to obtain \emph{some} cautious consequences.
On the other hand, the computation of cautious consequences is similar to the computation of backbones of propositional theories, for which anytime algorithms do exist.
We adapted one of these algorithms to cautious reasoning, showing that underestimates can be effectively obtained in reasonable time also for hard instances.
Moreover, we introduced a general strategy to obtain anytime variants of existing algorithms such as those implemented by DLV and clasp.
All algorithms as well as a proof-of-concept parallel implementation were implemented in the solver WASP.
Our empirical evaluation highlights that sound answers are computable within the first seconds of computation in many cases.
Moreover, the performance of the parallel system is encouraging and leaves space for future work on this subject.

\noindent \textbf{Acknowledgements.} This research has been partly supported by Regione Calabria under the EU Social Fund and project PIA KnowRex POR FESR 2007- 2013, 
by the Italian Ministry of University and Research under PON project ``Ba2Know S.I.-LAB'' n. PON03PE\_0001, and by National Group for Scientific Computation (GNCS-INDAM).


\appendix

\section{Proof of Theorem~\ref{thm:main}}

The proof is split into several lemmas using $P_i,L_i,U_i,O_i,I_i$ to denote the content of variables $P,L,U,O,I$ at step $i$ of computation ($i \geq 0$).
More in detail, in Lemma~\ref{lem:seq} we will first show that underestimates form an increasing sequence and, on the contrary, overestimates form a decreasing sequence.
Then, in Lemma~\ref{lem:sm} we will prove properties of stable models of programs $P_i \cup L_i$ ($i \geq 0$).
Correctness of estimates will be shown in Lemmas~\ref{lem:over}--\ref{lem:under}, and termination of the algorithms in Lemma~\ref{lem:terminate}.
Finally, in Lemma~\ref{lem:anytime} we will extend the proof to variants using ComputeStableModel$^*$.

\begin{lemma}\label{lem:seq}
$U_i \subseteq U_{i+1}$ and $O_{i+1} \subseteq O_i \subseteq Q$ for each $i \geq 0$.
\end{lemma}
\begin{proof}
Variable $U$ is initially empty.
EnumerationOfModels and OverestimateReduction reassign $U$ only once.
IterativeCoherenceTesting and UnderestimateReduction always enlarge the set stored in $U$ by means of set union (line~4).
Concerning variable $O$, it is initially equal to $Q$ and restricted at each reassignment by means of set intersection (line~7 for OverestimateReduction; line~6 for the other procedures).
\end{proof}

\begin{lemma}\label{lem:sm}
$SM(P_{i+1} \cup L_{i+1}) \subseteq SM(P_i \cup L_i)$ for each $i \geq 0$.
For IterativeCoherenceTesting and IterativePartialCoherenceTesting we also have $SM(P_{i+1} \cup L_{i+1}) = SM(P_i \cup L_i)$ for each $i \geq 0$.
\end{lemma}
\begin{proof}
Variable $P$ is reassigned only by EnumerationOfModels and OverestimateReduction, where constraints are added to the previous program.
Constraints can only remove stable models (as a consequence of the Splitting Set Theorem by \citeNP{lifs-turn-94}).
On the other hand, learned constraints stored in variable $L$ are implicit in the program stored by variable $P$, and thus cannot change its semantics.
\end{proof}

\begin{lemma}\label{lem:over}
$O_i \supseteq Q \cap CC(P)$ for each $i \geq 0$.
\end{lemma}
\begin{proof}
The base case is true because $O_0 = Q$.
Assume the claim is true for some $i \geq 0$ and consider $O_{i+1} = O_i \cap I_{i+1}$, where $I_{i+1} \in SM(P_i \cup L_i)$.
By $i$ applications of Lemma~\ref{lem:sm}, we obtain $I_{i+1} \in SM(P_0 \cup L_0)$, i.e., $I_{i+1} \in SM(P)$.
We can thus conclude $a \in O_i \setminus O_{i+1}$ implies $a \notin CC(P)$, and we are done.
\end{proof}

\begin{lemma}\label{lem:under}
$U_i \subseteq Q \cap CC(P)$ for each $i \geq 0$.
\end{lemma}
\begin{proof}
The base case is true because $U_0 = \emptyset$.
Assume the claim is true for some $i \geq 0$ and consider $U_{i+1}$.
If $U_{i+1} = U_i$ then the claim is true.
Otherwise, we distinguish two cases.

For IterativeCoherenceTesting and IterativePartialCoherenceTesting, $U_{i+1} = U_i \cup \{a\}$ for some $a \in O_i \setminus U_i$.
Moreover, there is no $M \in SM(P_i \cup L_i)$ such that $a \notin M$ because $I_{i+1} = \bot$.
From Lemma~\ref{lem:sm}, we can conclude that there is no $M \in SM(P)$ such that $a \notin M$, i.e., $a \in CC(P)$.
Since $a \in O_i \setminus U_i$, we have $a \in O_i$ and thus $a \in Q$ by Lemma~\ref{lem:seq}.
Therefore, $a \in Q \cap CC(P)$ and we are done.

For EnumerationOfModels and OverestimateReduction, $U_{i+1} = O_i$ and the algorithm terminates.
Exactly $i+1$ constraints were added to $P$, one for each stable model of $P$ found, i.e., $I_1, \ldots, I_i$.
Moreover, $I_{i+1} = \bot$ holds.
Assume by contradiction that there is $a \in O_i \setminus CC(P)$.
Hence, there is $M \in SM(P)$ such that $a \notin M$.
Moreover, $a \in I_j$ ($j = 1, \ldots, i$) and thus $M$ is a model of all constraints added at line~1.
Consequently, $M$ is a stable model of $P_i \cup L_i$, which contradicts $I_{i+1} = \bot$.
\end{proof}

\begin{lemma}\label{lem:terminate}
Algorithm~\ref{alg:cautious} 
terminates after finitely many steps.
\end{lemma}
\begin{proof}
When EnumerationOfModels is used, termination is guaranteed because $P$ has a finite number of stable models.
OverestimateReduction either sets $U$ equal to $O$, or reduces $O$, which initially is equal to $Q$, a finite set.
IterativeCoherenceTesting either increases $U$, or reduces $O$, and thus terminates because $O$ is finite and $U_i \subseteq O_i$ holds for each $i \geq 0$ by Lemmas~\ref{lem:over} and \ref{lem:under}.
Termination of IterativePartialCoherenceTesting is guaranteed if restarts are properly delayed during the computation, as it must be done already for guaranteeing termination of stable model search.
\end{proof}

\begin{lemma}\label{lem:anytime}
Underestimates produced by ComputeStableModel$^*$ are sound.
\end{lemma}
\begin{proof}
Follows by the fact that $L$ contains constraints that are implicit in the program stored by variable $P$.
\end{proof}

\label{lastpage}

\end{document}